\documentclass{article} 
\usepackage{iclr2017_workshop,times}
\usepackage{hyperref}
\usepackage{url}

\usepackage{amsmath}
\usepackage{amsfonts}
\usepackage{amsthm}
\usepackage{amssymb}
\usepackage{fixmath}
\usepackage{mathtools}

\usepackage{floatrow}
\newfloatcommand{capbtabbox}{table}[][\FBwidth]

 \usepackage{microtype}

\usepackage{times}
\usepackage{graphicx} 
\usepackage{subfig}

\usepackage{algorithm}
\usepackage{algorithmic}

\RequirePackage[singlelinecheck=off]{caption}

\newcommand{\norm}[1]{\left\lVert#1\right\rVert}

\newcommand{\rank}{r}
\DeclareMathOperator{\ttrank}{TT-rank}
\DeclareMathOperator{\ttround}{TT-round}

\renewcommand{\vec}[1]{\boldsymbol{#1}}
\newcommand{\tens}[1]{\boldsymbol{\mathcal{#1}}}
\newcommand{\tensel}[1]{\mathcal{#1}}

\newtheorem{theorem}{Theorem}
\newtheorem*{theorem*}{Theorem}
\newtheorem{lemma}{Lemma}

\graphicspath{{figures/}}
\newcommand{\executeiffilenewer}[3]{%
  \ifnum\pdfstrcmp{\pdffilemoddate{#1}}%
  {\pdffilemoddate{#2}} > 0 {\immediate\write18{#3}}\fi}

\title{Exponential Machines}

\author{Alexander Novikov$^{2,3}$\\
\texttt{novikov@bayesgroup.ru}\\
\And
Mikhail Trofimov$^4$\\
\texttt{mikhail.trofimov@phystech.edu}\\
\And
Ivan Oseledets$^{1,3}$\\
\texttt{i.oseledets@skoltech.ru}\\~\\
$^1$Skolkovo Institute of Science and Technology, Russian Federation\\
$^2$National Research University Higher School of Economics, Russian Federation\\
$^3$Institute of Numerical Mathematics, Russian Federation\\
$^4$Moscow Institute of Physics and Technology, Russian Federation
}

\begin{document}

\maketitle

\begin{abstract}
Modeling interactions between features improves the performance of machine learning solutions in many domains (e.g. recommender systems or sentiment analysis).
In this paper, we introduce Exponential Machines (ExM), a predictor that models all interactions of every order.
The key idea is to represent an exponentially large tensor of parameters in a factorized format called Tensor Train (TT).
The Tensor Train format regularizes the model and lets you control the number of underlying parameters.
To train the model, we develop a stochastic Riemannian optimization procedure, which allows us to fit tensors with $2^{160}$ entries.
We show that the model achieves state-of-the-art performance on synthetic data with high-order interactions and that it works on par with high-order factorization machines on a recommender system dataset MovieLens 100K.
\end{abstract}

\section{Introduction}
Machine learning problems with categorical data require modeling interactions between the features to solve them.
As an example, consider a sentiment analysis problem -- detecting whether a review is positive or negative -- and the following dataset: `I liked it', `I did not like it', `I'm not sure'.
 Judging by the presence of the word `like' or the word `not' alone, it is hard to understand the tone of the review. But the presence of the \emph{pair} of words `not' and `like' strongly indicates a negative opinion.

If the dictionary has $d$ words, modeling pairwise interactions requires~$O(d^2)$ parameters and will probably overfit to the data.
Taking into account all interactions (all pairs, triplets, etc. of words) requires impractical~$2^d$ parameters.

In this paper, we show a scalable way to account for all interactions. Our contributions are:
\begin{itemize}
\item We propose a predictor that models \emph{all} $2^d$ interactions of $d$-dimensional data by representing the exponentially large tensor of parameters in a compact multilinear format -- Tensor Train (TT-format)~(Sec.~\ref{sec:the-model}).
Factorizing the parameters into the TT-format leads to a better generalization, a linear with respect to $d$ number of underlying parameters and inference time (Sec.~\ref{sec:inference}).
The TT-format lets you control the number of underlying parameters through the \emph{TT-rank} -- a generalization of the matrix rank to tensors.
\item We develop a stochastic Riemannian optimization learning algorithm~(Sec.~\ref{sec:riemannian-optimization}).
In our experiments, it outperformed the stochastic gradient descent baseline~(Sec.~\ref{sec:exp-riemannian-optimization}) that is often used for models parametrized by a tensor decomposition~(see related works, Sec.~\ref{sec:related-work}).
\item We show that the linear model (e.g. logistic regression) is a special case of our model with the TT-rank equal $2$~(Sec.~\ref{sec:exp-initialization}).
\item We extend the model to handle interactions between functions of the features, not just between the features themselves~(Sec.~\ref{sec:model-extension}).
\end{itemize}

\section{Linear model}
In this section, we describe a generalization of a class of machine learning algorithms -- the \emph{linear model}.
Let us fix a training dataset of pairs $\{(\vec{x}^{(f)}, y^{(f)})\}_{f=1}^N$, where $\vec{x}^{(f)}$ is a $d$-dimensional feature vector of $f$-th object, and $y^{(f)}$ is the corresponding target variable.
Also fix a loss function $\ell(\widehat{y}, y):\mathbb{R}^2 \to \mathbb{R}$, which takes as input the predicted value $\widehat{y}$ and the ground truth value $y$.
We call a model \emph{linear}, if the prediction of the model depends on the features $\vec{x}$ only via the dot product between the features $\vec{x}$ and the $d$-dimensional vector of parameters $\vec{w}$:
\begin{equation}
\label{eq:linear-model}
\widehat{y}_{\text{linear}}(\vec{x}) = \langle \vec{x}, \vec{w} \rangle + b,
\end{equation}
where $b \in \mathbb{R}$ is the \emph{bias} parameter.

One of the approaches to learn the parameters $\vec{w}$ and $b$ of the model is to minimize the following loss
\begin{equation}
\label{eq:linear-loss}
\sum_{f=1}^N \ell\left(\langle \vec{x}^{(f)}, \vec{w} \rangle + b,\, y^{(f)}\right) + \frac{\lambda}{2} \norm{\vec{w}}^2_2,
\end{equation}
where $\lambda$ is the regularization parameter.
For the linear model we can choose any regularization term instead of $L_2$, but later the choice of the regularization term will become important~(see Sec.~\ref{sec:riemannian-optimization}).

Several machine learning algorithms can be viewed as a special case of the linear model with an appropriate choice of the loss function $\ell(\widehat{y}, y)$: least squares regression (squared loss), Support Vector Machine (hinge loss), and logistic regression (logistic loss).

\section{Our model \label{sec:the-model}}
Before introducing our model equation in the general case, consider a $3$-dimensional example. The equation includes one term per each subset of features (each interaction)
\begin{equation}
\label{eq:polynomial-model-example}
\begin{aligned}
\widehat{y}(\vec{x}) &= \tensel{W}_{000} + \tensel{W}_{100} \,\, x_1 + \tensel{W}_{010} \,\, x_2 + \tensel{W}_{001} x_3 \\
&+ \tensel{W}_{110} \,\, x_1 x_2 + \tensel{W}_{101} \,\, x_1 x_3 + \tensel{W}_{011} \,\, x_2 x_3 \\
&+ \tensel{W}_{111} \,\, x_1 x_2 x_3.
\end{aligned}
\end{equation}
Note that all permutations of features in a term (e.g. $x_1 x_2$ and $x_2 x_1$) correspond to a single term and have exactly one associated weight (e.g. $\tensel{W}_{110}$).

In the general case, we enumerate the subsets of features with a binary vector $(i_1, \ldots, i_d)$, where $i_k = 1$ if the $k$-th feature belongs to the subset. The model equation looks as follows
\begin{equation}
\label{eq:polynomial-model}
\widehat{y}(\vec{x}) = \sum_{i_1=0}^1 \ldots \sum_{i_d=0}^1 \tensel{W}_{i_1 \ldots i_d} \prod_{k=1}^d x_k^{i_k}.
\end{equation}
Here we assume that $0^0 = 1$.
The model is parametrized by a $d$-dimensional tensor $\tens{W}$, which consists of $2^d$ elements.

The model equation~\eqref{eq:polynomial-model} is linear with respect to the weight tensor $\tens{W}$.
To emphasize this fact and simplify the notation we rewrite the model equation~\eqref{eq:polynomial-model} as a tensor dot product $\widehat{y}(\vec{x}) = \langle \tens{X}, \tens{W} \rangle$, where the tensor $\tens{X}$ is defined as follows
\begin{equation}
\label{eq:X-definition}
\tensel{X}_{i_1 \ldots i_d} = \prod_{k=1}^d x_k^{i_k}.
\end{equation}
Note that there is no need in a separate bias term, since it is already included in the model as the weight tensor element $\tensel{W}_{0 \ldots 0}$ (see the model equation example~\eqref{eq:polynomial-model-example}).

The key idea of our method is to compactly represent the exponentially large tensor of parameters $\tens{W}$ in the Tensor Train format~\citep{oseledets2011ttMain}.

\section{Tensor Train}
A $d$-dimensional tensor $\tens{A}$ is said to be represented in the Tensor Train (TT) format~\citep{oseledets2011ttMain}, if each of its elements can be computed as the following product of $d-2$ matrices and $2$ vectors
\begin{equation}
  \label{eq:TT-format}
\tensel{A}_{i_1 \ldots i_d} = G_1[i_1] \ldots G_d[i_d],
\end{equation}
where for any $k = 2, \ldots, d-1$ and for any value of $i_k$, $G_k[i_k]$ is an $\rank \times \rank$ matrix, $G_1[i_1]$ is a $1 \times \rank$ vector and $G_d[i_d]$ is an $\rank \times 1$ vector (see Fig.~\ref{fig:TT}).
We refer to the collection of matrices $G_k$ corresponding to the same dimension $k$ (technically, a $3$-dimensional array) as the $k$-th \emph{TT-core}, where $k = 1,\ldots,d$.
The size $\rank$ of the slices $G_k[i_k]$ controls the trade-off between the representational power of the TT-format and computational efficiency of working with the tensor. We call $\rank$ the \emph{TT-rank} of the tensor $\tens{A}$.

An attractive property of the TT-format is the ability to perform algebraic operations on tensors without materializing them, i.e. by working with the TT-cores instead of the tensors themselves.
The TT-format supports computing the norm of a tensor and the dot product between tensors; element-wise sum and element-wise product of two tensors (the result is a tensor in the TT-format with increased TT-rank), and some other operations~\citep{oseledets2011ttMain}.

\begin{figure}[t]
    \captionsetup{singlelinecheck=on}
    \centering
	\resizebox{0.4\textwidth}{!}{
    \def\svgwidth{7cm}
    \normalsize
  \executeiffilenewer{figures/TT.svg}{figures/TT.pdf}%
  {inkscape -z -D --file=figures/TT.svg %
   --export-pdf=figures/TT.pdf --export-latex}%
  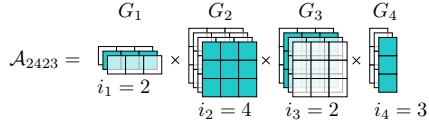%

    }
    \caption{An illustration of the TT-format for a $3 \times 4 \times 4 \times 3$ tensor $\tens{A}$  with the TT-rank equal $3$. \label{fig:TT}}
\end{figure}

\section{Inference \label{sec:inference}}
In this section, we return to the model proposed in Sec.~\ref{sec:the-model} and show how to compute the model equation~\eqref{eq:polynomial-model} in linear time.
To avoid the exponential complexity, we represent the weight tensor~$\tens{W}$ and the data tensor~$\tens{X}$~\eqref{eq:X-definition} in the TT-format.
The TT-ranks of these tensors determine the efficiency of the scheme.
During the learning, we initialize and optimize the tensor $\tens{W}$ in the TT-format and explicitly control its TT-rank.
The TT-rank of the tensor $\tens{X}$ always equals 1. Indeed, the following TT-cores give the exact representation of the tensor~$\tens{X}$
\begin{equation*}
G_k[i_k] = x_k^{i_k} \in \mathbb{R}^{1 \times 1}, ~~ k=1, \ldots, d.
\end{equation*}
The $k$-th core $G_k[i_k]$ is a $1 \times 1$ matrix for any value of $i_k \in \{0, 1\}$, hence the TT-rank of the tensor $\tens{X}$ equals~$1$.

Now that we have TT-representations of tensors $\tens{W}$ and $\tens{X}$, we can compute the model response $\widehat{y}(\vec{x}) = \langle \tens{X}, \tens{W} \rangle$ in linear time with respect to the number of features $d$.
\begin{theorem}
\label{thm:inference-time}
The model response $\widehat{y}(\vec{x})$ can be computed in $O(\rank^2 d)$ FLOPS, where $\rank$ is the TT-rank of the weight tensor~$\tens{W}$.
\end{theorem}

We refer the reader to Appendix~\ref{sec:app-inference-time-proof} where we propose an inference algorithm with $O(r^2 d)$ complexity and thus prove Theorem~\ref{thm:inference-time}.

The TT-rank of the weight tensor $\tens{W}$ is a hyper-parameter of our method and it controls the efficiency vs. flexibility trade-off.
A small TT-rank regularizes the model and yields fast learning and inference but restricts the set of possible tensors $\tens{W}$.
A sufficiently large TT-rank allows any value of the tensor $\tens{W}$ and effectively leaves us with the full polynomial model without any advantages of the TT-format.

\section{Learning \label{sec:learning}}
Learning the parameters of the proposed model corresponds to minimizing the loss under the TT-rank constraint:
\begin{equation}
\label{eq:TT-loss-minimization}
\begin{aligned}
& \underset{\tens{W}}{\text{minimize}}
& & L(\tens{W}), \\
& \text{subject to}
& & \ttrank(\tens{W}) = r_0,
\end{aligned}
\end{equation}
where the loss is defined as follows
\begin{equation}
\label{eq:TT-loss}
L(\tens{W}) = \sum_{f=1}^N \ell\left(\langle \tens{X}^{(f)}, \tens{W} \rangle,\, y^{(f)}\right) + \frac{\lambda}{2} \norm{\tens{W}}^2_F,~~~
\norm{\tens{W}}^2_F = \sum_{i_1=0}^1 \ldots \sum_{i_d=0}^1 \tensel{W}^2_{i_1 \ldots i_d}.
\end{equation}

We consider two approaches to solve problem~\eqref{eq:TT-loss-minimization}.
In a baseline approach, we optimize the objective $L(\tens{W})$ with the stochastic gradient descent applied to the underlying parameters of the TT-format of the tensor $\tens{W}$.

An alternative to the baseline is to perform gradient descent with respect to the tensor $\tens{W}$, that is subtract the  gradient from the current estimate of $\tens{W}$ on each iteration.
The TT-format indeed allows to subtract tensors, but this operation increases the TT-rank on each iteration, making this approach impractical.

To improve upon the baseline and avoid the TT-rank growth, we exploit the geometry of the set of tensors that satisfy the TT-rank constraint~\eqref{eq:TT-loss-minimization} to build a Riemannian optimization procedure (Sec.~\ref{sec:riemannian-optimization}).
We experimentally show the advantage of this approach over the baseline in Sec.~\ref{sec:exp-riemannian-optimization}.

\subsection{Riemannian optimization \label{sec:riemannian-optimization}}
The set of all $d$-dimensional tensors with fixed TT-rank $\rank$
\begin{equation*}
\mathcal{M}_r = \{\tens{W} \in \mathbb{R}^{2 \times  \ldots \times 2}\!:\, \ttrank(\tens{W})=\rank\}
\end{equation*}
forms a Riemannian manifold~\citep{holtz2012manifolds}.
This observation allows us to use Riemannian optimization to solve problem~\eqref{eq:TT-loss-minimization}.
Riemannian gradient descent consists of the following steps which are repeated until convergence (see Fig.~\ref{fig:riemannian-illustration} for an illustration):
\begin{enumerate}
\item Project the gradient $\frac{\partial L}{\partial \tens{W}}$ on the tangent space of $\mathcal{M}_r$ taken at the point $\tens{W}$. We denote the tangent space as $T_{\tens{W}} \mathcal{M}_r$ and the projection as $\tens{G} = P_{T_{\tens{W}} \mathcal{M}_r}(\frac{\partial L}{\partial \tens{W}})$.
\item Follow along $\tens{G}$ with some step~$\alpha$ (this operation increases the TT-rank).
\item Retract the new point $\tens{W} - \alpha \tens{G}$ back to the manifold $\mathcal{M}_r$, that is decrease its TT-rank to $\rank$.
\end{enumerate}
We now describe how to implement each of the steps outlined above.

\cite{lubich2015time} proposed an algorithm to project a TT-tensor $\tens{Z}$ on the tangent space of $\mathcal{M}_r$ at a point $\tens{W}$ which consists of two steps: preprocess $\tens{W}$ in $O(d \rank^3)$ and project $\tens{Z}$ in $O(d \rank^2 \ttrank(\tens{Z})^2)$.
\cite{lubich2015time} also showed that the TT-rank of the projection is bounded by a constant that is independent of the TT-rank of the tensor $\tens{Z}$:
\begin{equation*}
   \ttrank(P_{T_{\tens{W}} \mathcal{M}_r}(\tens{Z})) \leq 2 \ttrank(\tens{W}) = 2 \rank.
\end{equation*}

Let us consider the gradient of the loss function~\eqref{eq:TT-loss}
\begin{equation}
\label{eq:TT-gradient}
\frac{\partial L}{\partial \tens{W}} = \sum_{f=1}^N \frac{\partial \ell}{\partial \widehat{y}} \tens{X}^{(f)} + \lambda \tens{W}.
\end{equation}

Using the fact that $P_{T_{\tens{W}} \mathcal{M}_r}(\tens{W}) = \tens{W}$ and that the projection is a linear operator we get
\begin{equation}
\label{eq:riemannian-gradient}
P_{T_{\tens{W}} \mathcal{M}_r}\left ( \frac{\partial L}{\partial \tens{W}} \right) = \sum_{f=1}^N \frac{\partial \ell}{\partial \widehat{y}} P_{T_{\tens{W}} \mathcal{M}_r}(\tens{X}^{(f)}) + \lambda \tens{W}.
\end{equation}
Since the resulting expression is a weighted sum of projections of individual data tensors $\tens{X}^{(f)}$, we can project them  in parallel. Since the TT-rank of each of them equals 1 (see Sec.~\ref{sec:inference}), all $N$ projections cost $O(d \rank^2 (\rank + N))$ in total.
The TT-rank of the projected gradient is less or equal to $2 \rank$ regardless of the dataset size~$N$.

Note that here we used the particular choice of the regularization term. For terms other than $L_2$ (e.g. $L_1$), the gradient may have arbitrary large TT-rank.

As a retraction -- a way to return back to the manifold $\mathcal{M}_r$ -- we use the TT-rounding procedure \citep{oseledets2011ttMain}. For a given tensor~$\tens{W}$ and rank~$\rank$ the TT-rounding procedure returns a tensor $\widehat{\tens{W}} = \ttround(\tens{W},\,\rank)$ such that its TT-rank equals $\rank$ and the Frobenius norm of the residual~$\| \tens{W} - \widehat{\tens{W}}\|_F$ is as small as possible.
The computational complexity of the TT-rounding procedure is $O(dr^3)$.

Since we aim for big datasets, we use a stochastic version of the Riemannian gradient descent: on each iteration we sample a random mini-batch of objects from the dataset, compute the stochastic gradient for this mini-batch, make a step along the projection of the stochastic gradient, and retract back to the manifold (Alg.~\ref{alg:rimeannian-optimization}).

An iteration of the stochastic Riemannian gradient descent consists of inference $O(dr^2M)$, projection $O(dr^2(r + M))$, and retraction $O(dr^3)$, which yields $O(dr^2(r + M))$ total computational complexity.
\begin{figure}[t]
\begin{minipage}{.61\textwidth}
  \centering
  \begin{algorithm}[H]
     \caption{Riemannian optimization}
     \label{alg:rimeannian-optimization}
  \begin{algorithmic}
     \STATE {\bfseries Input:} Dataset $\{(\vec{x}^{(f)}, y^{(f)})\}_{f=1}^N$, desired TT-rank $\rank_0$, number of iterations~$T$, mini-batch size~$M$, learning rate~$\alpha$, regularization strength~$\lambda$
     \STATE {\bfseries Output:} $\tens{W}$ that approximately minimizes \eqref{eq:TT-loss-minimization}
     \STATE Train linear model~\eqref{eq:linear-loss} to get the parameters $\vec{w}$ and $b$
     \STATE Initialize the tensor $\tens{W}_0$ from $\vec{w}$ and $b$ with the TT-rank equal $\rank_0$
     \FOR{$t := 1$ {\bfseries to} $T$}
          \STATE Sample $M$ indices $h_1, \ldots, h_M \sim \mathcal{U}(\{1, \ldots, N\})$
          \STATE $\tens{D}_t := \sum_{j=1}^M \frac{\partial \ell}{\partial \widehat{y}} \tens{X}^{(h_j)}  + \lambda \tens{W}_{t-1}$
          \STATE $\tens{G}_t := P_{T_{\tens{W}_{t-1}} \mathcal{M}_r}\left ( \tens{D}_t \right)$~\eqref{eq:riemannian-gradient}
          \STATE $\tens{W}_{t} := \ttround(\tens{W}_{t-1} - \alpha \tens{G}_t, \,r_0)$
     \ENDFOR
  \end{algorithmic}
  \end{algorithm}
\end{minipage}%
\hfill
\begin{minipage}[t]{.35\textwidth}
  \centering
  \vspace{-2.7cm}
  \begin{figure}[H]
      \centering
      \resizebox{0.85\textwidth}{!}{
      \def\svgwidth{7cm}
      \normalsize
  \executeiffilenewer{figures/manifold.svg}{figures/manifold.pdf}%
  {inkscape -z -D --file=figures/manifold.svg %
   --export-pdf=figures/manifold.pdf --export-latex}%
  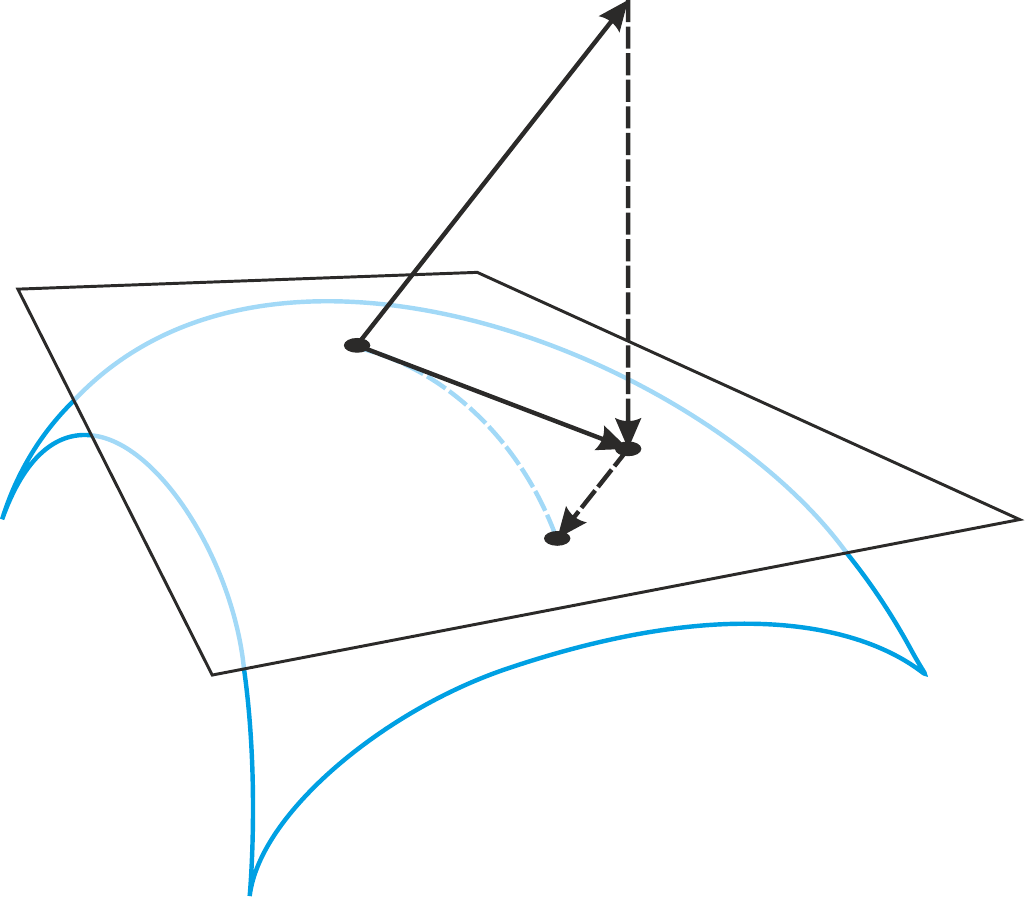%

      }
      \caption{An illustration of one step of the Riemannian gradient descent. The step-size $\alpha$ is assumed to be $1$ for clarity of the figure. \label{fig:riemannian-illustration}}
    \end{figure}
\end{minipage}
\end{figure}

\subsection{Initialization \label{sec:initialization}}
We found that a random initialization for the TT-tensor $\tens{W}$ sometimes freezes the convergence of optimization method (see Sec.~\ref{sec:exp-initialization}).
We propose to initialize the optimization from the solution of the corresponding linear model~\eqref{eq:linear-model}.

The following theorem shows how to initialize the weight tensor $\tens{W}$ from a linear model.
\begin{theorem}
\label{thm:initialization-rank}
For any $d$-dimensional vector $\vec{w}$ and a bias term $b$ there exist a tensor $\tens{W}$ of TT-rank $2$, such that for any $d$-dimensional vector $\vec{x}$ and the corresponding object-tensor $\tens{X}$ the dot products $\langle \vec{x}, \vec{w} \rangle$ and $\langle \tens{X}, \tens{W} \rangle$ coincide.
\end{theorem}
The proof is provided in Appendix~\ref{sec:app-initialization-rank-proof}.

\section{Extending the model \label{sec:model-extension}}
In this section, we extend the proposed model to handle polynomials of any functions of the features.
As an example, consider the logarithms of the features in the $2$-dimensional case:
\begin{equation*}
  \begin{aligned}
    \widehat{y}^{\,\log}(\vec{x}) =& \,\tensel{W}_{00} + \tensel{W}_{01} x_1 +\tensel{W}_{10} x_2 + \tensel{W}_{11} x_1 x_2 + \tensel{W}_{20} \,\, \log(x_1) + \tensel{W}_{02} \,\, \log(x_2)\\
    &+ \tensel{W}_{12} \,\, x_1 \log(x_2)+ \tensel{W}_{21} \,\, x_2 \log(x_1)
    + \tensel{W}_{22} \,\, \log(x_1) \log(x_2).
  \end{aligned}
\end{equation*}
In the general case, to model interactions between $n_g$ functions $g_1, \ldots, g_{n_g}$ of the features we redefine the object-tensor as follows:
\begin{equation*}
\tensel{X}_{i_1 \ldots i_d} = \prod_{k=1}^d c(x_k, i_k),
\end{equation*}
where
\begin{equation*}
c(x_k, i_k) =
\begin{cases}
1, & \text{if } i_k = 0,\\
g_1(x_k), & \text{if } i_k = 1,\\
\ldots\\
g_{n_g}(x_k), & \text{if } i_k = n_g,\\
\end{cases}
\end{equation*}

The weight tensor $\tens{W}$ and the object-tensor $\tens{X}$ are now consist of $(n_g + 1)^d$ elements.
After this change to the object-tensor $\tens{X}$, learning and inference algorithms will stay unchanged compared to the original model~\eqref{eq:polynomial-model}.

\paragraph{Categorical features.} Our basic model handles categorical features $x_k \in \{1, \ldots, K\}$ by converting them into one-hot vectors $x_{k,1}, \ldots, x_{k,K}$. The downside of this approach is that it wastes the model capacity on modeling non-existing interactions between the one-hot vector elements $x_{k,1}, \ldots, x_{k,K}$ which correspond to the same categorical feature. Instead, we propose to use one TT-core per categorical feature and use the model extension technique with the following function
\begin{equation*}
c(x_k, i_k) =
\begin{cases}
1, & \text{if } x_k = i_k \text{ or } i_k = 0,\\
0, & \text{otherwise.}
\end{cases}
\end{equation*}
This allows us to cut the number of parameters per categorical feature from $2K\rank^2$ to $(K + 1)\rank^2$ without losing any representational power.

%
%
%

\section{Experiments}
We release a Python implementation of the proposed algorithm and the code to reproduce the experiments\footnote{\url{https://github.com/Bihaqo/exp-machines}}.
For the operations related to the TT-format, we used the TT-Toolbox\footnote{\url{https://github.com/oseledets/ttpy}}.

\subsection{Datasets \label{sec:exp-datasets}}
The datasets used in the experiments (see details in Appendix~\ref{sec:app-datasets})
\begin{enumerate}
  \item \textbf{UCI~\citep{Lichman2013UCI} Car dataset} is a classification problem with $1728$ objects and $21$ binary features (after one-hot encoding).
  We randomly splitted the data into $1382$ training and $346$ test objects and binarized the labels for simplicity.

  \item \textbf{UCI HIV dataset} is a binary classification problem with $1625$ objects and $160$ features, which we randomly splitted into $1300$ training and $325$ test objects.

  \item \textbf{Synthetic data.} We generated $100\,000$ train and $100\,000$ test objects with $30$ features and set the ground truth target variable to a $6$-degree polynomial of the features.

  \item \textbf{MovieLens 100K} is a recommender system dataset with $943$ users and $1682$ movies~\citep{harper2015movielens}.
  We followed~\cite{blondel2016hofm} in preparing $2703$ one-hot features and in turning the problem into binary classification.
\end{enumerate}

\subsection{Riemannian optimization \label{sec:exp-riemannian-optimization}}
In this experiment, we compared two approaches to training the model: Riemannian optimization~(Sec.~\ref{sec:riemannian-optimization}) vs. the baseline~(Sec.~\ref{sec:learning}).
In this and later experiments we tuned the learning rate of both Riemannian and SGD optimizers with respect to the training loss after 100 iterations by the grid search with logarithmic grid.

On the Car and HIV datasets we turned off the regularization ($\lambda = 0$) and used rank $r = 4$.
We report that on the Car dataset Riemannian optimization (learning rate $\alpha = 40$) converges faster and achieves better final point than the baseline (learning rate $\alpha = 0.03$) both in terms of the training and test losses~(Fig.~\ref{fig:riemannian_vs_plain}a,~\ref{fig:riemannian-vs-plain-val}a).
On the HIV dataset Riemannian optimization (learning rate $\alpha = 800$) converges to the value $10^{-4}$ around $20$ times faster than the baseline~(learning rate $\alpha = 0.001$, see Fig.~\ref{fig:riemannian_vs_plain}b), but the model overfitts to the data~(Fig.~\ref{fig:riemannian-vs-plain-val}b).

The results on the synthetic dataset with high-order interactions confirm the superiority of the Riemannian approach over SGD -- we failed to train the model at all with SGD~(Fig.~\ref{fig:riemannian-vs-plain-synthetic}).

\begin{figure}[t]
  \captionsetup[subfigure]{justification=centering}
  \centering

  \subfloat[Binarized Car dataset]{%
    \includegraphics[clip,width=0.5\textwidth]{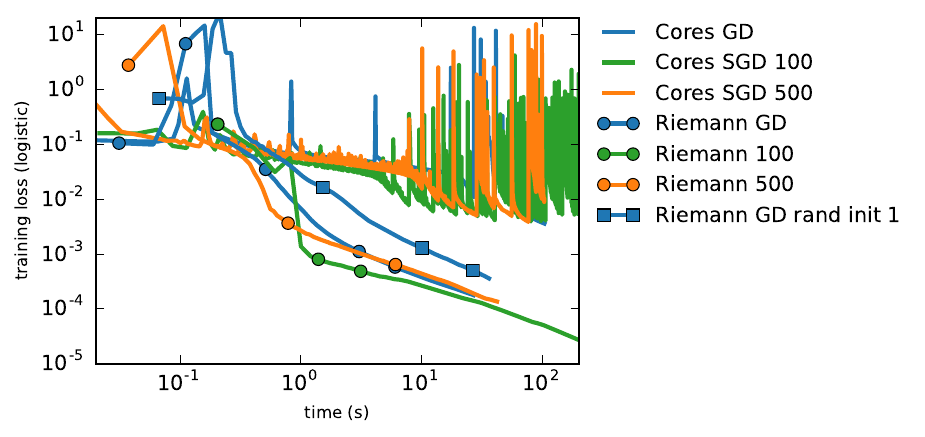}%
  }
  \subfloat[HIV dataset]{%
    \includegraphics[clip,width=0.5\textwidth]{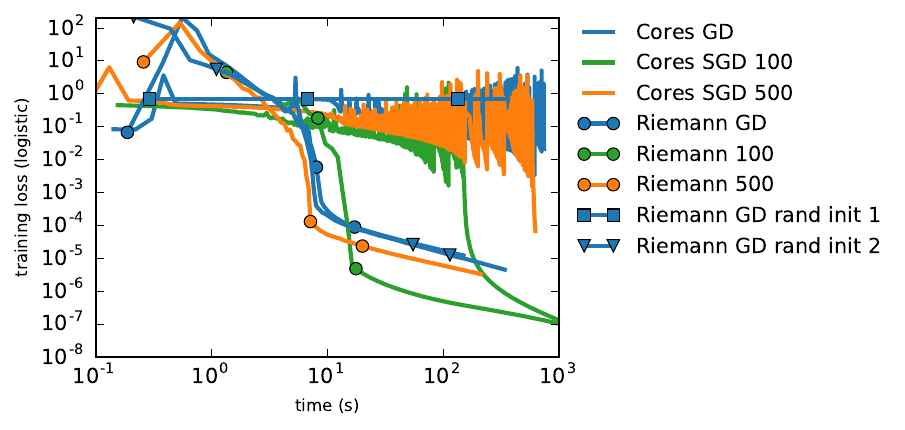}%
  }
  \caption{A comparison between Riemannian optimization and SGD applied to the underlying parameters of the TT-format (the baseline) for the rank-$4$ Exponential Machines.
  Numbers in the legend stand for the batch size.
  The methods marked with `rand init' in the legend (square and triangle markers) were initialized from a random TT-tensor from two different distributions (see Sec.~\ref{sec:exp-initialization}), all other methods were initialized from the solution of ordinary linear logistic regression.
  Type-2 random initialization is ommited from the Car dataset for the clarity of the figure. \label{fig:riemannian_vs_plain}}
\end{figure}

On the MovieLens 100K dataset, we have only used SGD-type algorithms, because using the one-hot feature encoding is much slower than using the categorical version (see Sec.~\ref{sec:model-extension}), and we have yet to implement the support for categorical features for the Riemannian optimizer. On the bright side, prototyping the categorical version of ExM in TensorFlow allowed us to use a GPU accelerator.

\subsection{Initialization \label{sec:exp-initialization}}
In this experiment, we compared random initialization with the initialization from the solution of the corresponding linear problem~(Sec.~\ref{sec:initialization}).
We explored two ways to randomly initialize a TT-tensor: 1) filling its TT-cores with independent Gaussian noise; 2) initializing $\tens{W}$ to represent a linear model with random coefficients (sampled from a standard Gaussian).
We report that on the Car dataset type-1 random initialization slowed the convergence compared to initialization from the linear model solution~(Fig.~\ref{fig:riemannian_vs_plain}a), while on the HIV dataset the convergence was completely frozen~(Fig.~\ref{fig:riemannian_vs_plain}b).

Two possible reasons for this effect are:
a) the vanishing and exploding gradients problem~\citep{bengio1994vanishing} that arises when dealing with a product of a large number of factors ($160$ in the case of the HIV dataset);
b) initializing the model in such a way that high-order terms dominate we may force the gradient-based optimization to focus on high-order terms, while it may be more stable to start with low-order terms instead.
Type-2 initialization (a random linear model) indeed worked on par with the best linear initialization on the Car, HIV, and synthetic datasets~(Fig.~\ref{fig:riemannian_vs_plain}b,~\ref{fig:riemannian-vs-plain-synthetic}).

\subsection{Comparison to other approaches \label{sec:exp-comparison}}
On the synthetic dataset with high-order interactions we compared Exponential Machines (the proposed method) with scikit-learn implementation~\citep{scikit-learn} of logistic regression, random forest, and kernel SVM; FastFM implementation~\citep{bayer2015fastfm} of $2$-nd order Factorization Machines; our implementation of high-order Factorization Machines\footnote{\url{https://github.com/geffy/tffm}}; and a feed-forward neural network implemented in TensorFlow~\citep{tensorflow2015-whitepaper}.
We used $6$-th order FM with the Adam optimizer~\citep{kingma2014adam} for which we had chosen the best rank ($20$) and learning rate ($0.003$) based on the training loss after the first $50$ iterations.
We tried several feed-forward neural networks with ReLU activations and up to $4$ fully-connected layers and $128$ hidden units.
We compared the models based on the Area Under the Curve (AUC) metric since it is applicable to all methods and is robust to unbalanced labels~(Tbl.~\ref{tbl:synthetic-comparison}).


\begin{figure}
\begin{floatrow}
  \capbtabbox{
  \scalebox{0.9}{
  \begin{tabular}{l*{3}{c}}
  \hline\\[-0.3cm]
  Method & Test AUC &
  \parbox{1.4cm}{Training\\ time (s)} &
  \parbox{1.4cm}{Inference\\ time (s)} \\[0.2cm]
  \hline
  Log. reg. & $0.50$ & $0.4$ & $0.0$ \\
  RF & $0.55$ & $21.4$ & $6.5$ \\
  Neural Network & $0.50$ & $47.2$ & $0.1$ \\
  SVM RBF & $0.50$ & $2262.6$ & $5380$ \\
  SVM poly. 2 & $0.50$ & $1152.6$ & $4260$ \\
  SVM poly. 6 & $0.56$ & $4090.9$ & $3774$ \\
  2-nd order FM & $0.50$ & $638.2$ & $0.5$ \\
  6-th order FM & $0.57$ & $549$ & $3$ \\
  6-th order FM & $0.86$ & $6039$ & $3$ \\
  6-th order FM & $\textbf{0.96}$ & $38918$ & $3$ \\
  ExM rank 3 & $0.79$ & $65$ & $0.2$ \\
  ExM rank 8 & $0.85$ & $1831$ & $1.3$ \\
  ExM rank 16 & $\textbf{0.96}$ & $48879$ & $3.8$ \\
  \hline
  \end{tabular}}
  }{
  \caption{A comparison between models on synthetic data with high-order interactions (Sec.~\ref{sec:exp-comparison}).
  We report the inference time on $100000$ test objects in the last column.
  \label{tbl:synthetic-comparison}}
  }
\ffigbox{%
      \hspace{-1.3cm}
      \includegraphics[width=0.38\textwidth]{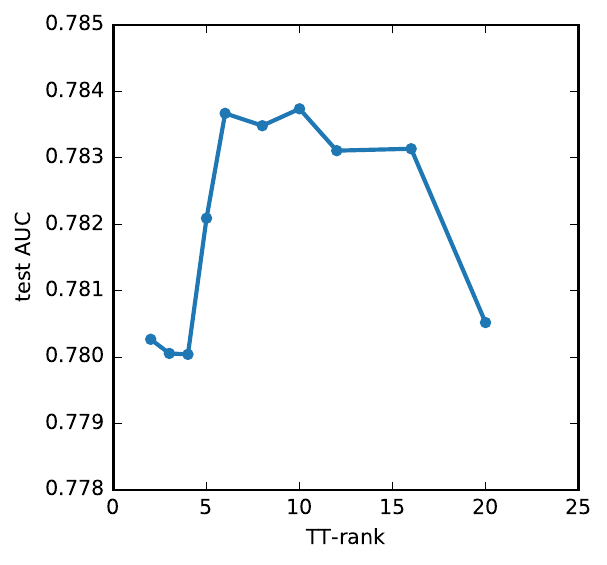}
}{%
      \caption{The influence of the TT-rank on\\
      the test AUC for the MovieLens 100K\\
       dataset. \label{fig:rank-vs-accuracy}}
}
\end{floatrow}
\vspace{-0.4cm}
\end{figure}

On the MovieLens 100K dataset we used the categorical features representation described in Sec.~\ref{sec:model-extension}.
Our model obtained $0.784$ test AUC with the TT-rank equal $10$ in $273$ seconds on a Tesla K40 GPU (the inference time is $0.3$ seconds per $78800$ test objects); our implentation of 3-rd order FM obtained $0.782$; logistic regression obtained $0.782$; and \cite{blondel2016hofm} reported $0.786$ with 3-rd order FM on the same data.

\subsection{TT-rank \label{sec:exp-tt-rank}}
The TT-rank is one of the main hyperparameters of the proposed model.
Two possible strategies can be used to choose it: grid-search or DMRG-like algorithms (see Sec.~\ref{sec:related-work}).
In our experiments we opted for the former and observed that the model is fairly robust to the choice of the TT-rank (see Fig.~\ref{fig:rank-vs-accuracy}), but a too small TT-rank can hurt the accuracy (see Tbl.~\ref{tbl:synthetic-comparison}).

\section{Related work \label{sec:related-work}}
Kernel SVM is a flexible non-linear predictor and, in particular, it can model interactions when used with the polynomial kernel~\citep{boser1992training}.
As a downside, it scales at least quadratically with the dataset size~\citep{bordes2005fast} and overfits on highly sparse data.

With this in mind, \cite{rendle2010factorization}~developed Factorization Machine (FM), a general predictor that models pairwise interactions.
To overcome the problems of polynomial SVM, FM restricts the rank of the weight matrix, which leads to a linear number of parameters and generalizes better on sparse data.
FM running time is linear with respect to the number of nonzero elements in the data, which allows scaling to billions of training entries on sparse problems.

For high-order interactions FM uses CP-format~\citep{caroll70, harshman1970parafac} to represent the tensor of parameters.
The choice of the tensor factorization is the main difference between the high-order FM and Exponential Machines.
The TT-format comes with two advantages over the CP-format: first, the TT-format allows for Riemannian optimization; second, the problem of finding the best TT-rank $\rank$ approximation to a given tensor always has a solution and can be solved in polynomial time.
We found Riemannian optimization superior to the SGD baseline~(Sec.~\ref{sec:learning}) that was used in several other models parametrized by a tensor factorization~\citep{rendle2010factorization, lebedev2014speeding, novikov15tensornet}.
Note that CP-format also allows for Riemannian optimization, but only for 2-order tensors (and thereafter 2-order FM).

A number of works used full-batch or stochastic Riemannian optimization for data processing tasks~\citep{meyer2011regression, tan2014riemannian, xu2016stochastic, zhang2016fast}.
The last work~\citep{zhang2016fast} is especially interesting in the context of our method, since it improves the convergence rate of stochastic Riemannian gradient descent and is directly applicable to our learning procedure.

In a concurrent work,~\cite{stoudenmire2016quantum} proposed a model that is similar to ours but relies on the trigonometric basis $(\cos (\frac{\pi}{2} x), \sin (\frac{\pi}{2} x))$ in contrast to polynomials $(1, x)$ used in Exponential Machines (see Sec.~\ref{sec:model-extension} for an explanation on how to change the basis).
They also proposed a different learning procedure inspired by the DMRG algorithm~\citep{schollwock2011dmrg}, which allows to automatically choose the ranks of the model, but is hard to adapt to the stochastic regime.
One of the possible ways to combine strengths of the DMRG and Riemannian approaches is to do a full DMRG sweep once in a few epochs of the stochastic Riemannian gradient descent to adjust the ranks.

Other relevant works include the model that approximates the decision function with a multidimensional Fourier series whose coefficients lie in the TT-format~\citep{wahls2014learning}; and models that are similar to FM but include squares and other powers of the features: Tensor Machines~\citep{yang2015tensor} and Polynomial Networks~\citep{livni2014computational}.
Tensor Machines also enjoy a theoretical generalization bound.
In another relevant work, \cite{blondel2016polynomial} boosted the efficiency of FM and Polynomial Networks by casting their training as a low-rank tensor estimation problem, thus making it multi-convex and allowing for efficient use of Alternative Least Squares types of algorithms.
Note that Exponential Machines are inherently multi-convex.

\section{Discussion}
We presented a predictor that models all interactions of every order.
To regularize the model and to make the learning and inference feasible, we represented the exponentially large tensor of parameters in the Tensor Train format.
To train the model, we used Riemannian optimization in the stochastic regime and report that it outperforms a popular baseline based on the stochastic gradient descent.
However, the Riemannian learning algorithm does not support sparse data, so for dataset with hundreds of thousands of features we are forced to fall back on the baseline learning method.
We found that training process is sensitive to initialization and proposed an initialization strategy based on the solution of the corresponding linear problem.
The solutions developed in this paper for the stochastic Riemannian optimization may suit other machine learning models parametrized by tensors in the TT-format.

\subsubsection*{Acknowledgments}
The study has been funded by the Russian Academic Excellence Project `5-100'.

\bibliography{paper}

\begin{thebibliography}{28}
\providecommand{\natexlab}[1]{#1}
\providecommand{\url}[1]{\texttt{#1}}
\expandafter\ifx\csname urlstyle\endcsname\relax
  \providecommand{\doi}[1]{doi: #1}\else
  \providecommand{\doi}{doi: \begingroup \urlstyle{rm}\Url}\fi

\bibitem[Abadi et~al.(2015)Abadi, Agarwal, Barham, Brevdo, Chen, Citro,
  Corrado, Davis, Dean, Devin, Ghemawat, Goodfellow, Harp, Irving, Isard, Jia,
  Jozefowicz, Kaiser, Kudlur, Levenberg, Man\'{e}, Monga, Moore, Murray, Olah,
  Schuster, Shlens, Steiner, Sutskever, Talwar, Tucker, Vanhoucke, Vasudevan,
  Vi\'{e}gas, Vinyals, Warden, Wattenberg, Wicke, Yu, and
  Zheng]{tensorflow2015-whitepaper}
Mart\'{\i}n Abadi, Ashish Agarwal, Paul Barham, Eugene Brevdo, Zhifeng Chen,
  Craig Citro, Greg~S. Corrado, Andy Davis, Jeffrey Dean, Matthieu Devin,
  Sanjay Ghemawat, Ian Goodfellow, Andrew Harp, Geoffrey Irving, Michael Isard,
  Yangqing Jia, Rafal Jozefowicz, Lukasz Kaiser, Manjunath Kudlur, Josh
  Levenberg, Dan Man\'{e}, Rajat Monga, Sherry Moore, Derek Murray, Chris Olah,
  Mike Schuster, Jonathon Shlens, Benoit Steiner, Ilya Sutskever, Kunal Talwar,
  Paul Tucker, Vincent Vanhoucke, Vijay Vasudevan, Fernanda Vi\'{e}gas, Oriol
  Vinyals, Pete Warden, Martin Wattenberg, Martin Wicke, Yuan Yu, and Xiaoqiang
  Zheng.
\newblock {TensorFlow}: Large-scale machine learning on heterogeneous systems,
  2015.
\newblock URL \url{http://tensorflow.org/}.
\newblock Software available from tensorflow.org.

\bibitem[Bayer(2015)]{bayer2015fastfm}
I.~Bayer.
\newblock Fastfm: a library for factorization machines.
\newblock \emph{arXiv preprint arXiv:1505.00641}, 2015.

\bibitem[Bengio et~al.(1994)Bengio, Simard, and Frasconi]{bengio1994vanishing}
Y.~Bengio, P.~Simard, and P.~Frasconi.
\newblock Learning long-term dependencies with gradient descent is difficult.
\newblock \emph{IEEE transactions on neural networks}, 5\penalty0 (2):\penalty0
  157--166, 1994.

\bibitem[Blondel et~al.(2016{\natexlab{a}})Blondel, Fujino, Ueda, and
  Ishihata]{blondel2016hofm}
M.~Blondel, A.~Fujino, N.~Ueda, and M.~Ishihata.
\newblock Higher-order factorization machines.
\newblock 2016{\natexlab{a}}.

\bibitem[Blondel et~al.(2016{\natexlab{b}})Blondel, Ishihata, Fujino, and
  Ueda]{blondel2016polynomial}
M.~Blondel, M.~Ishihata, A.~Fujino, and N.~Ueda.
\newblock Polynomial networks and factorization machines: New insights and
  efficient training algorithms.
\newblock In \emph{Advances in Neural Information Processing Systems 29
  (NIPS)}. 2016{\natexlab{b}}.

\bibitem[Bordes et~al.(2005)Bordes, Ertekin, Weston, and
  Bottou]{bordes2005fast}
A.~Bordes, S.~Ertekin, J.~Weston, and L.~Bottou.
\newblock Fast kernel classifiers with online and active learning.
\newblock \emph{The Journal of Machine Learning Research}, 6:\penalty0
  1579--1619, 2005.

\bibitem[Boser et~al.(1992)Boser, Guyon, and Vapnik]{boser1992training}
B.~E. Boser, I.~M. Guyon, and V.~N. Vapnik.
\newblock A training algorithm for optimal margin classifiers.
\newblock In \emph{Proceedings of the fifth annual workshop on Computational
  learning theory}, pp.\  144--152, 1992.

\bibitem[Caroll \& Chang(1970)Caroll and Chang]{caroll70}
J.~D. Caroll and J.~J. Chang.
\newblock Analysis of individual differences in multidimensional scaling via
  n-way generalization of {E}ckart-{Y}oung decomposition.
\newblock \emph{Psychometrika}, 35:\penalty0 283--319, 1970.

\bibitem[Harper \& Konstan(2015)Harper and Konstan]{harper2015movielens}
F.~M. Harper and A.~J. Konstan.
\newblock The movielens datasets: History and context.
\newblock \emph{ACM Transactions on Interactive Intelligent Systems (TiiS)},
  2015.

\bibitem[Harshman(1970)]{harshman1970parafac}
R.~A. Harshman.
\newblock Foundations of the {PARAFAC} procedure: models and conditions for an
  explanatory multimodal factor analysis.
\newblock \emph{UCLA Working Papers in Phonetics}, 16:\penalty0 1--84, 1970.

\bibitem[Holtz et~al.(2012)Holtz, Rohwedder, and Schneider]{holtz2012manifolds}
S.~Holtz, T.~Rohwedder, and R.~Schneider.
\newblock On manifolds of tensors of fixed {TT}-rank.
\newblock \emph{Numerische Mathematik}, pp.\  701--731, 2012.

\bibitem[Kingma \& Ba(2014)Kingma and Ba]{kingma2014adam}
D.~Kingma and J.~Ba.
\newblock Adam: A method for stochastic optimization.
\newblock \emph{arXiv preprint arXiv:1412.6980}, 2014.

\bibitem[Lebedev et~al.(2014)Lebedev, Ganin, Rakhuba, Oseledets, and
  Lempitsky]{lebedev2014speeding}
V.~Lebedev, Y.~Ganin, M.~Rakhuba, I.~Oseledets, and V.~Lempitsky.
\newblock Speeding-up convolutional neural networks using fine-tuned
  {CP}-decomposition.
\newblock In \emph{International Conference on Learning Representations
  (ICLR)}, 2014.

\bibitem[Lichman(2013)]{Lichman2013UCI}
M.~Lichman.
\newblock {UCI} machine learning repository, 2013.

\bibitem[Livni et~al.(2014)Livni, Shalev-Shwartz, and
  Shamir]{livni2014computational}
R.~Livni, S.~Shalev-Shwartz, and O.~Shamir.
\newblock On the computational efficiency of training neural networks.
\newblock In \emph{Advances in Neural Information Processing Systems 27
  (NIPS)}, 2014.

\bibitem[Lubich et~al.(2015)Lubich, Oseledets, and
  Vandereycken]{lubich2015time}
C.~Lubich, I.~V. Oseledets, and B.~Vandereycken.
\newblock Time integration of tensor trains.
\newblock \emph{SIAM Journal on Numerical Analysis}, pp.\  917--941, 2015.

\bibitem[Meyer et~al.(2011)Meyer, Bonnabel, and Sepulchre]{meyer2011regression}
G.~Meyer, S.~Bonnabel, and R.~Sepulchre.
\newblock Regression on fixed-rank positive semidefinite matrices: a
  {R}iemannian approach.
\newblock \emph{The Journal of Machine Learning Research}, pp.\  593--625,
  2011.

\bibitem[Novikov et~al.(2015)Novikov, Podoprikhin, Osokin, and
  Vetrov]{novikov15tensornet}
A.~Novikov, D.~Podoprikhin, A.~Osokin, and D.~Vetrov.
\newblock Tensorizing neural networks.
\newblock In \emph{Advances in Neural Information Processing Systems 28
  (NIPS)}. 2015.

\bibitem[Oseledets(2011)]{oseledets2011ttMain}
I.~V. Oseledets.
\newblock {T}ensor-{T}rain decomposition.
\newblock \emph{SIAM J. Scientific Computing}, 33\penalty0 (5):\penalty0
  2295--2317, 2011.

\bibitem[Pedregosa et~al.(2011)Pedregosa, Varoquaux, Gramfort, Michel, Thirion,
  Grisel, Blondel, Prettenhofer, Weiss, Dubourg, Vanderplas, Passos,
  Cournapeau, Brucher, Perrot, and Duchesnay]{scikit-learn}
F.~Pedregosa, G.~Varoquaux, A.~Gramfort, V.~Michel, B.~Thirion, O.~Grisel,
  M.~Blondel, P.~Prettenhofer, R.~Weiss, V.~Dubourg, J.~Vanderplas, A.~Passos,
  D.~Cournapeau, M.~Brucher, M.~Perrot, and E.~Duchesnay.
\newblock Scikit-learn: Machine learning in {P}ython.
\newblock \emph{Journal of Machine Learning Research}, 12:\penalty0 2825--2830,
  2011.

\bibitem[Rendle(2010)]{rendle2010factorization}
S.~Rendle.
\newblock Factorization machines.
\newblock In \emph{Data Mining (ICDM), 2010 IEEE 10th International Conference
  on}, pp.\  995--1000, 2010.

\bibitem[Schollw{\"o}ck(2011)]{schollwock2011dmrg}
U.~Schollw{\"o}ck.
\newblock The density-matrix renormalization group in the age of matrix product
  states.
\newblock \emph{Annals of Physics}, 326\penalty0 (1):\penalty0 96--192, 2011.

\bibitem[Stoudenmire \& Schwab(2016)Stoudenmire and
  Schwab]{stoudenmire2016quantum}
E.~Stoudenmire and D.~J. Schwab.
\newblock Supervised learning with tensor networks.
\newblock In \emph{Advances in Neural Information Processing Systems 29
  (NIPS)}. 2016.

\bibitem[Tan et~al.(2014)Tan, Tsang, Wang, Vandereycken, and
  Pan]{tan2014riemannian}
M.~Tan, I.~W. Tsang, L.~Wang, B.~Vandereycken, and S.~J. Pan.
\newblock Riemannian pursuit for big matrix recovery.
\newblock 2014.

\bibitem[Wahls et~al.(2014)Wahls, Koivunen, Poor, and
  Verhaegen]{wahls2014learning}
S.~Wahls, V.~Koivunen, H.~V. Poor, and M.~Verhaegen.
\newblock Learning multidimensional fourier series with tensor trains.
\newblock In \emph{Signal and Information Processing (GlobalSIP), 2014 IEEE
  Global Conference on}, pp.\  394--398. IEEE, 2014.

\bibitem[Xu \& Ke(2016)Xu and Ke]{xu2016stochastic}
Z.~Xu and Y.~Ke.
\newblock Stochastic variance reduced {R}iemannian eigensolver.
\newblock \emph{arXiv preprint arXiv:1605.08233}, 2016.

\bibitem[Yang \& Gittens(2015)Yang and Gittens]{yang2015tensor}
J.~Yang and A.~Gittens.
\newblock Tensor machines for learning target-specific polynomial features.
\newblock \emph{arXiv preprint arXiv:1504.01697}, 2015.

\bibitem[Zhang et~al.(2016)Zhang, Reddi, and Sra]{zhang2016fast}
H.~Zhang, S.~J. Reddi, and S.~Sra.
\newblock Fast stochastic optimization on {R}iemannian manifolds.
\newblock \emph{arXiv preprint arXiv:1605.07147}, 2016.

\end{thebibliography}
\bibliographystyle{iclr2017_conference}

\appendix
\addcontentsline{toc}{section}{Appendices}

\section{Proof of Theorem~\ref{thm:inference-time} \label{sec:app-inference-time-proof}}
Theorem~\ref{thm:inference-time} states that the inference complexity of the proposed algorithm is~$O(\rank^2 d)$, where $\rank$ is the TT-rank of the weight tensor~$\tens{W}$. In this section, we propose an algorithm that achieve the stated complexity and thus prove the theorem.
\begin{proof}
  Let us rewrite the definition of the model response~\eqref{eq:polynomial-model} assuming that the weight tensor $\tens{W}$ is represented in the TT-format~\eqref{eq:TT-format}
\begin{equation*}
\widehat{y}(\vec{x}) = \sum_{i_1, \ldots, i_d} \tensel{W}_{i_1 \ldots i_d} \, \left ( \prod_{k=1}^d x_k^{i_k} \right )
= \sum_{i_1, \ldots, i_d} G_1[i_1] \ldots G_d[i_d] \left ( \prod_{k=1}^d x_k^{i_k} \right ).
\end{equation*}
Let us group the factors that depend on the variable $i_k$, $k=1, \ldots, d$
\begin{equation*}
\begin{aligned}
\widehat{y}(\vec{x}) &= \sum_{i_1, \ldots, i_d}  x_1^{i_1} G_1[i_1] \ldots x_d^{i_d} G_d[i_d]
= \left ( \sum_{i_1=0}^1 x_1^{i_1} G_1[i_1] \right ) \ldots \left ( \sum_{i_d=0}^1 x_d^{i_d} G_d[i_d] \right )\\
&= \underbrace{A_1}_{1 \times \rank} \underbrace{A_2}_{\rank \times \rank} \ldots \underbrace{A_d}_{\rank \times 1},
\end{aligned}
\end{equation*}
where the matrices $A_k$ for $k=1, \ldots, d$ are defined as follows
\begin{equation*}
A_k = \sum_{i_k=0}^1 x_k^{i_k} G_k[i_k] = G_k[0] + x_k G_k[1].
\end{equation*}
The final value $\widehat{y}(\vec{x})$ can be computed from the matrices~$A_k$ via $d-1$ matrix-by-vector multiplications and $1$ vector-by-vector multiplication, which yields $O(r^2 d)$ complexity.

Note that the proof is constructive and corresponds to the implementation of the inference algorithm.
\end{proof}

\begin{figure}[t]
  \captionsetup[subfigure]{justification=centering}
  \centering
  \subfloat[Binarized Car dataset]{%
    \includegraphics[clip,width=0.5\textwidth]{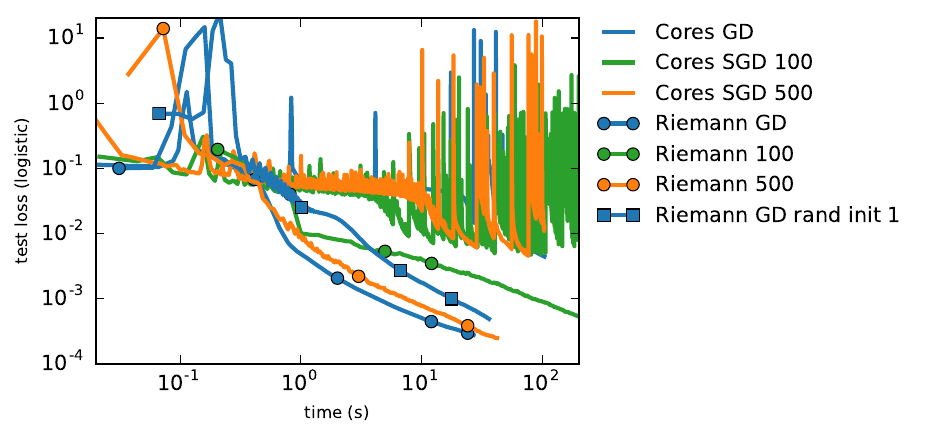}
  }
  \subfloat[HIV dataset]{%
    \includegraphics[clip,width=0.5\textwidth]{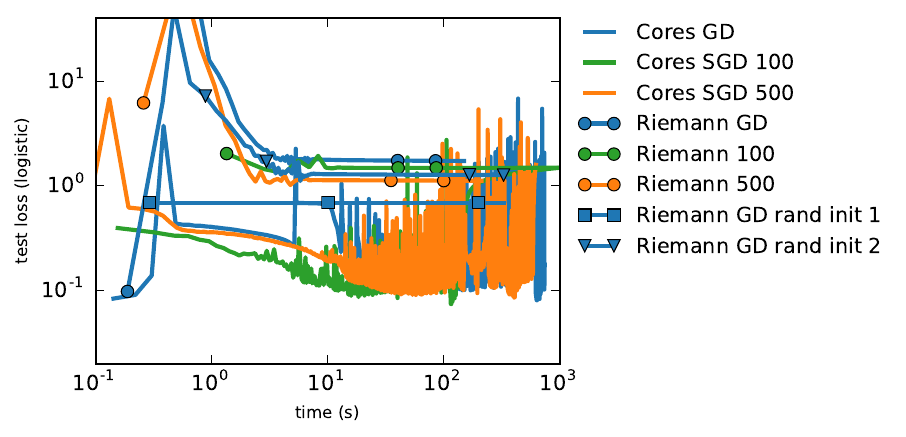}
  }
\caption{A comparison between Riemannian optimization and SGD applied to the underlying parameters of the TT-format (the baseline) for the rank-$4$ Exponential Machines.
Numbers in the legend stand for the batch size.
The methods marked with `rand init' in the legend (square and triangle markers) were initialized from a random TT-tensor from two different distributions, all other methods were initialized from the solution of ordinary linear logistic regression.
See details in Sec.~\ref{sec:exp-riemannian-optimization} and~\ref{sec:exp-initialization} \label{fig:riemannian-vs-plain-val}}
\end{figure}
\section{Proof of Theorem~\ref{thm:initialization-rank} \label{sec:app-initialization-rank-proof}}
Theorem~\ref{thm:initialization-rank} states that it is possible to initialize the weight tensor~$\tens{W}$ of the proposed model from the weights $\vec{w}$ of the linear model.
\begin{theorem*}
For any $d$-dimensional vector $\vec{w}$ and a bias term $b$ there exist a tensor $\tens{W}$ of TT-rank $2$, such that for any $d$-dimensional vector $\vec{x}$ and the corresponding object-tensor $\tens{X}$ the dot products $\langle \vec{x}, \vec{w} \rangle$ and $\langle \tens{X}, \tens{W} \rangle$ coincide.
\end{theorem*}

To prove the theorem, in the rest of this section we show that the tensor $\tens{W}$ from Theorem~\ref{thm:initialization-rank} is representable in the TT-format with the following TT-cores
\begin{equation}
\begin{aligned}
\label{eq:initialization-cores}
G_1[0] &=
\left[
\begin{array}{cc}
1 & 0\\
\end{array}
\right], ~~&
G_1[1] &=
\left[
\begin{array}{cc}
0 & w_1\\
\end{array}
\right], \\
G_d[0] &=
\left[
\begin{array}{c}
b\\
1\\
\end{array}
\right], ~~&
G_d[1] &=
\left[
\begin{array}{c}
w_d\\
0\\
\end{array}
\right], ~~\\[0.2cm]
\forall ~ 2 \leq & ~k \leq d-1 & &\\[-0.05cm]
G_k[0] &=
\left[
\begin{array}{cc}
1 & 0 \\
0 & 1 \\
\end{array}
\right], ~~&
G_k[1] &=
\left[
\begin{array}{cc}
0 & w_k \\
0 & 0 \\
\end{array}
\right],
\end{aligned}
\end{equation}
and thus the TT-rank of the tensor $\tens{W}$ equals $2$.

We start the proof with the following lemma:
\begin{lemma}
  \label{lemma:initialization-rank}
  For the TT-cores~\eqref{eq:initialization-cores} and any $p = 1, \ldots, d-1$ the following invariance holds:
  \begin{equation*}
  G_1[i_1] \ldots G_p[i_p] =
  \begin{cases*}
      \begin{array}{ll}
      \left[
      \begin{array}{cc}
      1 & 0\\
      \end{array}
      \right], & \text{if } \sum_{q=1}^p i_q = 0,\\[0.1cm]
      \left[
      \begin{array}{cc}
      0 & 0\\
      \end{array}
      \right], & \text{if } \sum_{q=1}^p i_q \geq 2,\\[0.1cm]
      \left[
      \begin{array}{cc}
      0 & w_k\\
      \end{array}
      \right], & \text{if } \sum_{q=1}^p i_q = 1,\\
                          &~~~~~~\!\text{and}~i_k = 1.
  \end{array}
  \end{cases*}
  \end{equation*}
\end{lemma}

\begin{proof}
  We prove the lemma by induction. Indeed, for $p = 1$ the statement of the lemma becomes
  \begin{equation*}
    G_1[i_1] =
    \begin{cases*}
        \begin{array}{ll}
        \left[
        \begin{array}{cc}
        1 & 0\\
        \end{array}
        \right], & \text{if } i_1 = 0,\\[0.1cm]
        \left[
        \begin{array}{cc}
        0 & w_1\\
        \end{array}
        \right], & \text{if } i_1 = 1,
    \end{array}
    \end{cases*}
  \end{equation*}
  which holds by definition of the first TT-core $G_1[i_1]$.

  Now suppose that the statement of Lemma~\ref{lemma:initialization-rank} is true for some $p - 1 \geq 1$.
  If $i_p = 0$, then $G_p[i_p]$ is an identity matrix and $G_1[i_1] \ldots G_p[i_p] = G_1[i_1] \ldots G_{p-1}[i_{p-1}]$. Also, $\sum_{q=1}^p i_q = \sum_{q=1}^{p-1}  i_q$, so the statement of the lemma stays the same.

  If $i_p = 1$, then there are $3$ options:
  \begin{itemize}
    \item If $\sum_{q=1}^{p-1} i_q = 0$, then $\sum_{q=1}^{p} i_q = 1$ and
    \begin{equation*}
      G_1[i_1] \ldots G_p[i_p] \,=\, \left[
      \begin{array}{cc}
      1 & 0\\
      \end{array}
      \right]
      G_p[1] \,=\, \left[
      \begin{array}{cc}
      0 & w_p\\
      \end{array}
      \right].
    \end{equation*}
    \item If $\sum_{q=1}^{p-1} i_q \geq 2$, then $\sum_{q=1}^{p} i_q\geq 2$ and
    \begin{equation*}
      G_1[i_1] \ldots G_p[i_p] \,=\, \left[
      \begin{array}{cc}
      0 & 0\\
      \end{array}
      \right] G_p[1] \,=\,
      \left[
      \begin{array}{cc}
      0 & 0\\
      \end{array}
      \right].
    \end{equation*}
    \item If $\sum_{q=1}^{p-1}i_q = 1$ with $i_k = 1$, then $\sum_{q=1}^{p}i_q \geq 2$ and
    \begin{equation*}
      G_1[i_1] \ldots G_p[i_p] \,=\, \left[
      \begin{array}{cc}
      0 & w_k\\
      \end{array}
      \right]
      G_p[1] \, =\, \left[
      \begin{array}{cc}
      0 & 0\\
      \end{array}
      \right].
    \end{equation*}
  \end{itemize}
  Which is exactly the statement of Lemma~\ref{lemma:initialization-rank}.
\end{proof}

\begin{figure}[t]
  \captionsetup[subfigure]{justification=centering}
  \centering
  \subfloat[Training set]{%
    \includegraphics[clip,width=0.5\textwidth]{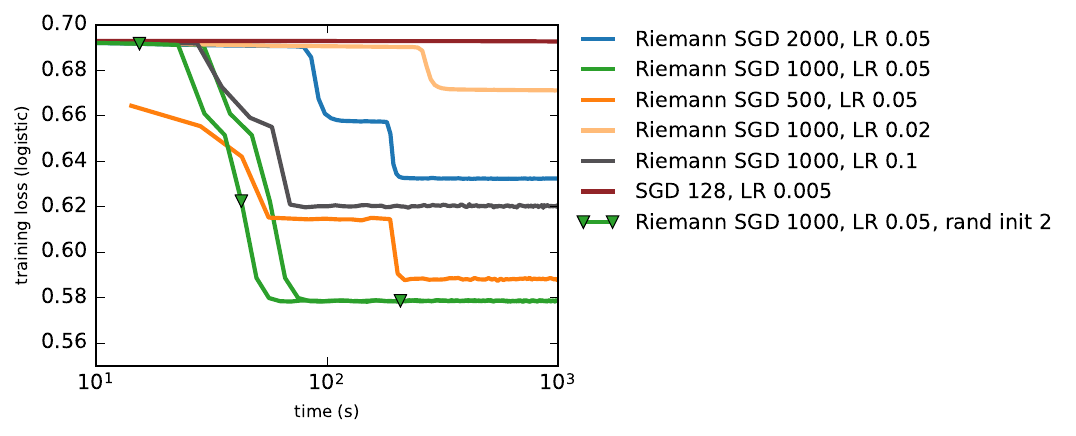}
  }
  \subfloat[Test set]{%
    \includegraphics[clip,width=0.5\textwidth]{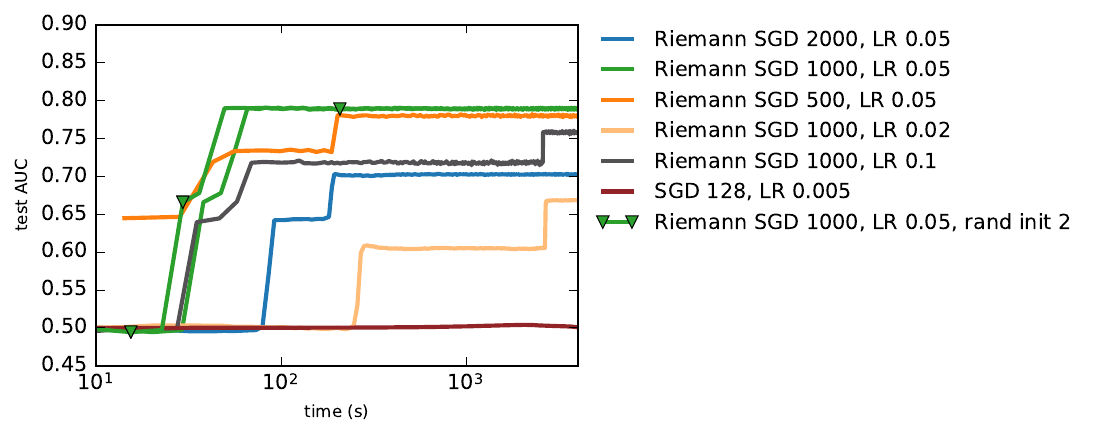}
  }
\caption{A comparison between Riemannian optimization and SGD applied to the underlying parameters of the TT-format (the baseline) for the rank-$3$ Exponential Machines on the synthetic dataset with high order interactions.
The first number in each legend enrty stands for the batch size.
The method marked with `rand init' in the legend (triangle markers) was initialized from a random linear model, all other methods were initialized from the solution of ordinary linear logistic regression.
See details in Sec.~\ref{sec:exp-riemannian-optimization} and~\ref{sec:exp-initialization} \label{fig:riemannian-vs-plain-synthetic}}
\end{figure}
\begin{proof}[Proof of Theorem \ref{thm:initialization-rank}]
The product of all TT-cores can be represented as a product of the first $p = d-1$ cores times the last core $G_d[i_d]$ and by using Lemma~\ref{lemma:initialization-rank} we get
\begin{equation*}
\tensel{W}_{i_1 \ldots i_d} = G_1[i_1] \ldots G_{d-1}[i_{d-1}] G_{d}[i_{d}]
	=~~
	\begin{cases*}
	    \begin{array}{ll}
	    b, & \text{if } \sum_{q=1}^d i_q = 0,\\[0.1cm]
	    0, & \text{if } \sum_{q=1}^d i_q \geq 2,\\[0.1cm]
	    w_k, & \text{if } \sum_{q=1}^d i_q = 1,\\
	                        &~~~~~\,\text{and}~i_k = 1.
	\end{array}
	\end{cases*}
\end{equation*}
The elements of the obtained tensor $\tens{W}$ that correspond to interactions of order $\geq 2$ equal to zero; the weight that corresponds to $x_k$ equals to $w_k$; and the bias term $\tensel{W}_{0 \ldots 0} = b$.

The TT-rank of the obtained tensor equal $2$ since its TT-cores are of size $2 \times 2$.
\end{proof}

\section{Detailed description of the datasets \label{sec:app-datasets}}
We used the following datasets for the experimental evaluation
\begin{enumerate}
  \item \textbf{UCI~\citep{Lichman2013UCI} Car dataset} is a classification problem with $1728$ objects and $21$ binary features (after one-hot encoding).
  We randomly splitted the data into $1382$ training and $346$ test objects.
  For simplicity, we binarized the labels: we picked the first class (`unacc') and made a one-versus-rest binary classification problem from the original Car dataset.

  \item \textbf{UCI~\citep{Lichman2013UCI} HIV dataset} is a binary classification problem with $1625$ objects and $160$ features, which we randomly splitted into $1300$ training and $325$ test objects.

  \item \textbf{Synthetic data.} We generated $100\,000$ train and $100\,000$ test objects with $30$ features.
  Each entry of the data matrix $X$ was independently sampled from $\{-1, +1\}$ with equal probabilities $0.5$.
  We also uniformly sampled $20$ subsets of features (interactions) of order $6$: $j^1_1, \ldots, j^1_{6}, \ldots, j^{20}_1, \ldots, j^{20}_{6} \sim \mathcal{U}\{1, \ldots, 30\}$.
  We set the ground truth target variable to a deterministic function of the input: $y(\vec{x}) = \sum_{z=1}^{20} \varepsilon_z \prod_{h=1}^{6} x_{j^z_h}$, and sampled the weights of the interactions from the uniform distribution: $\varepsilon_1, \ldots, \varepsilon_{20} \sim \mathcal{U}(-1, 1)$.

  \item \textbf{MovieLens 100K.} MovieLens 100K is a recommender system dataset with $943$ users and $1682$ movies~\citep{harper2015movielens}.
  We followed~\cite{blondel2016hofm} in preparing the features and in turning the problem into binary classification.
  For users, we converted age (rounded to decades), living area (the first digit of the zipcode), gender and occupation into a binary indicator vector using one-hot encoding.
  For movies, we used the release year (rounded to decades) and genres, also encoded.
  This process yielded $49 + 29 = 78$ additional one-hot features for each user-movie pair ($943 + 1682 + 78$ features in total).
  Original ratings were binarized using $5$ as a threshold.
  This results in $21200$ positive samples, half of which were used for traininig (with equal amount of sampled negative examples) and the rest were used for testing.
\end{enumerate}

\end{document}